\pdfoutput=1

\documentclass[11pt]{article}
\usepackage[table]{xcolor}

\usepackage{ACL_format/acl}

\usepackage{times}
\usepackage{latexsym}
\usepackage{amsmath}
\usepackage{amssymb}
\usepackage{mathtools}
\usepackage{amsthm}
\usepackage{algorithm}
\usepackage{algorithmic} % Use algpseudocode package
\usepackage{multirow}
\usepackage{booktabs} % For \toprule, \midrule, etc.
\usepackage{subcaption}
\usepackage{enumitem}

\usepackage[T1]{fontenc}
\usepackage[utf8]{inputenc}
\usepackage{microtype}
\usepackage{inconsolata}
\usepackage{graphicx}

\usepackage{amsthm} 
\newtheorem{theorem}{Theorem}
\newtheorem{proposition}{Proposition}

\theoremstyle{definition}
\newtheorem{definition}{Definition}

\theoremstyle{remark}

\title{STAMP: Selective Task-Aware Mechanism for Text Privacy} 

\author{
  Fengwei Tian$^1$, Payel Bhattacharjee$^1$, Heidi Hanson$^2$, \\
  \textbf{Geoffrey D. Rubin$^1$, Joseph Y. Lo$^3$, Ravi Tandon$^1$} \\
  \\
  $^1$University of Arizona, 
  $^2$Oak Ridge National Laboratory,  
  $^3$Duke University \\
  \texttt{\{fengtian, payelb, grubin, tandonr\}@arizona.edu,}\\
  \texttt{hansonha@ornl.gov, joseph.lo@duke.edu}
}

\begin{document}

\maketitle

\begin{abstract}
    We present \textsc{STAMP} (Selective Task-Aware Mechanism for Text Privacy), a new framework for task-aware text privatization that achieves an improved privacy–utility trade-off. \textsc{STAMP} selectively allocates privacy budgets across tokens by jointly considering
    (i) each token’s importance to the downstream task (as measured via a task- or query-specific representation), and
    (ii) its privacy sensitivity (e.g., names, dates, identifiers).
    This token-level partitioning enables fine-grained, group-wise control over the level of noise applied to different parts of the input, balancing privacy protection with task relevance.
    To privatize individual token embeddings, we introduce the polar mechanism, which perturbs only the direction of embeddings on the unit sphere while preserving their magnitude. Decoding is performed via cosine nearest-neighbor search, aligning the perturbation geometry with the decoding geometry. Unlike isotropic noise mechanisms, the polar mechanism maintains semantic neighborhoods in the embedding space and better preserves downstream utility.
    Experimental evaluations on SQuAD, Yelp, and AG News datasets demonstrate that \textsc{STAMP}, when combined with the normalized polar mechanism, consistently achieves superior privacy–utility trade-offs across varying per-token privacy budgets.    
        
\end{abstract}

\section{Introduction}
    Modern large language models (LLMs) routinely operate on user-supplied text that may contain identifying or otherwise privacy-sensitive content.
    Practical deployments therefore require client-side protection mechanisms that preserve task utility while preventing the unintended disclosure of sensitive text within the input \cite{yan2024protecting, pan2020privacy}.
    This setting encompasses both \textit{inference-time context privacy}—where the user’s text is privatized immediately before being sent to a remote model—and \textit{privacy-preserving rewriting}, where the text is locally transformed prior to transmission, storage, sharing, or other downstream use. 
    In this work, we adopt the framework of local differential privacy (LDP)~\cite{arachchige2019local}, wherein randomization occurs locally at the user side, ensuring that the server, model owner, or any downstream observer only sees a privatized version of the text.
    Informally, the released text should not allow an observer to reliably determine whether any specific token was present in the original text.
    \vspace{5pt}
    
    \noindent \textbf{\textit{The Need for Selective, Task-Aware Privacy}.} 
    Prior approaches to local text privatization face several fundamental limitations.
    Classical randomized response \cite{warner1965randomized} achieves privacy by randomly replacing each input token with another from the same domain.
    In natural language, such random substitutions often produce unnatural or incoherent text, severely degrading utility.
    Adding coordinate-wise Laplace noise (or isotropic Gaussian noise) \cite{feyisetan2021private, feyisetan2020privacy} to embedding vectors is similarly problematic: semantic embeddings are not uniformly sensitive—small perturbations in some directions can flip meanings, while large perturbations in others have negligible effect.
    Restricting noise to direction-only perturbations helps preserve magnitude \cite{weggenmann2021differential}, but when applied uniformly and decoded using mismatched rules, it still distorts fine-grained semantic relations.
    Applying uniform privacy budgets across all tokens compounds the problem, as it perturbs both innocuous tokens and semantically crucial ones with equal intensity.
    
    \begin{figure*}[t]
        \centering
        \includegraphics[scale = 0.24]{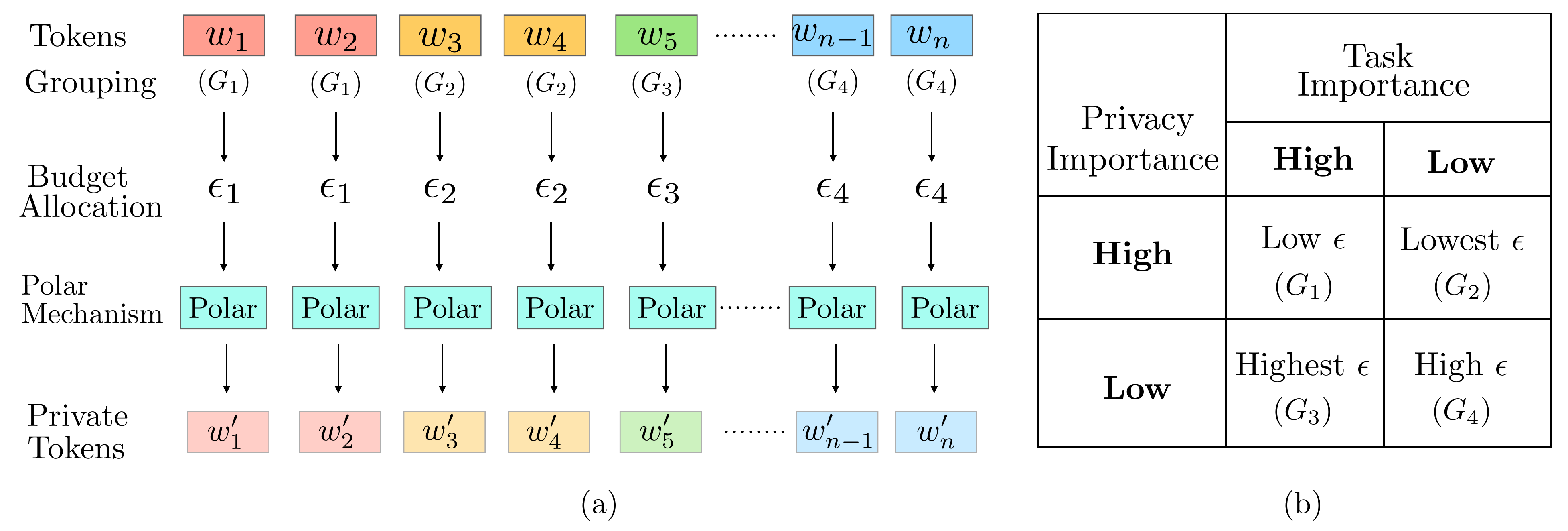}
        \caption{\textbf{Overview of the \textsc{STAMP}  framework}: Tokens are categorized according to their task and privacy relevance, and then perturbed using adaptively assigned privacy budgets via the polar mechanism. Figure (a) (left) illustrates the overall token perturbation pipeline with group-wise privacy budget allocation; (b) (right) details the grouping and budget assignment process based on task and privacy importance.}
        \label{tab:stamp_grid}
    \end{figure*}

    \noindent A principled mechanism must instead be selective, tailoring privacy to each \textit{token’s sensitivity} and \textit{importance to the task}. 
    Prior work has attempted selectivity using intrinsic linguistic heuristics—such as frequency, part of speech, or type-level information content \cite{meisenbacher2024thinking, meisenbacher2024spend} —but these ignore the downstream task context.
    A token may be essential for one query yet irrelevant for another, so its importance varies dynamically across contexts.
    Such static partitions misallocate privacy, leading to unnecessary noise or insufficient protection.
    For example, in a question-answering task, the token ``\texttt{Einstein}” is essential when answering ``\texttt{Who developed the theory of relativity?}” but irrelevant when the query is ``\texttt{When was the Nobel Prize established?}”.
    In contrast, in a customer-support context, a token such as ``\texttt{Alice Johnson}” or ``\texttt{ID: 123-45-6789}” may be highly sensitive yet irrelevant to the task objective, warranting stronger perturbation despite its low task importance.
    
%As illustrated in Figure~\ref{fig:teaser}, sensitive and task-relevant spans warrant tighter protection, whereas incidental tokens can tolerate greater perturbation.

    \noindent \textbf{Summary of Contributions.}  
    To address the above challenges, this work introduces \textsc{\textbf{STAMP}} (\emph{Selective Task-Aware Mechanism for Privacy}), 
    a framework for \emph{task-aware text privatization} that provides fine-grained control over privacy--utility trade-offs.  
    Our key contributions are as follows:  
    
    \begin{enumerate}[leftmargin=*, topsep=1pt, itemsep=3pt]
        \item \textbf{Selective, task-aware privacy allocation.}  
        \textsc{Stamp} partitions tokens based on two complementary dimensions:  
        (i) their \emph{importance to the downstream task}, derived from task- or query-specific representations, and  
        (ii) their \emph{privacy sensitivity}, such as the presence of identifiers, names, or dates.  
        This joint assessment enables \emph{group-wise privacy budgeting} that balances utility preservation and privacy protection.  
    
        \item \textbf{Geometry-aligned perturbation via the Polar Mechanism.}  
        We introduce the \emph{polar mechanism}, which privatizes token embeddings by perturbing only their \emph{direction} on the unit sphere while preserving magnitude.  
        Decoding is performed using \emph{cosine nearest-neighbor search}, aligning the perturbation geometry with the decoding geometry.  
        Unlike isotropic Gaussian or Laplace noise, this approach preserves \emph{semantic neighborhoods} and minimizes distortion of meaning.  
    
        \item \textbf{Comprehensive empirical evaluation.}  
        Experiments on \emph{SQuAD}, \emph{Yelp}, and \emph{AG News} demonstrate that \textsc{Stamp}, combined with the normalized polar mechanism, consistently achieves \emph{superior privacy--utility trade-offs} across a range of per-token privacy budgets.  
    \end{enumerate}

\section{Preliminaries and Related Work}
\label{sec:prelim-related}
\subsection{Preliminaries on LDP for Text Inputs}  
    
We consider the problem of \emph{inference-time context privacy}, where a user wishes to privatize their text input before sending it to a remote model (e.g., an LLM) for inference.   Given a user-supplied context \(c = (w_1, \ldots, w_n)\) consisting of tokens from a finite vocabulary \(\mathcal{V}\), a \emph{local privatization mechanism} \(\mathcal{M}\) operates to produce a privatized version \(\tilde{c}\).  
The server or model owner observes only \(\tilde{c}\), not the original context.  The objective is to preserve task utility for the user while limiting what can be inferred about the presence or absence of sensitive tokens or spans in \(c\).  This formulation also naturally extends to \emph{privacy-preserving text rewriting}, where the same mechanism is applied before text storage, sharing, or other downstream use.  

Each token \(w \in \mathcal{V}\) has an embedding \(e(w) \in \mathbb{R}^d\), and we denote its unit-normalized direction by \(\hat{e}(w) = e(w) / \|e(w)\|_2\).  
A token-level mechanism is a mapping \(\mathcal{M}: \mathcal{V} \to \mathcal{Y}\) that outputs a randomized token before transmission.  
For a context \(c = (w_1, \ldots, w_n)\), we write  
\[
\mathcal{M}^{(n)}(c) = \big(\mathcal{M}(w_1), \ldots, \mathcal{M}(w_n)\big),
\]
indicating that the mechanism \(\mathcal{M}\) is applied independently to each token.  

\begin{definition}[$(\epsilon,\delta)$-LDP (token level)]
\label{def:ldp-token}
Let \(\mathcal{V}\) be the token vocabulary and \(\mathcal{Y}\) the output space (e.g., privatized tokens).  
A randomized mechanism \(\mathcal{M}:\mathcal{V}\to\mathcal{Y}\) satisfies \((\epsilon,\delta)\)\emph{-local differential privacy (LDP)} if, for every pair of tokens \(w, w'\in\mathcal{V}\) and every measurable set \(S\subseteq\mathcal{Y}\),
\[
\Pr[\mathcal{M}(w)\in S] \le e^{\epsilon}\Pr[\mathcal{M}(w')\in S] + \delta,
\]
where \(\epsilon \ge 0\) is the privacy budget and \(\delta\) is a small failure probability.
\end{definition}

\noindent This definition ensures that observing the privatized token does not allow an adversary to reliably infer whether the original input was \(w\) or \(w'\); smaller \(\epsilon\) (and small \(\delta\)) imply stronger privacy guarantees.

\medskip
\noindent To better capture linguistic similarity, we also consider a distance-aware relaxation—\emph{metric LDP}—which scales indistinguishability according to a metric \(d\) between tokens:

\begin{definition}[$(\epsilon,\delta)$-metric LDP (token level)]
\label{def:metric-ldp-token}
Let \(\mathcal{V}\) be the token vocabulary and \(d:\mathcal{V}\times\mathcal{V}\to\mathbb{R}_{\ge 0}\) a metric on tokens.  
A randomized mechanism \(\mathcal{M}:\mathcal{V}\to\mathcal{Y}\) satisfies \((\epsilon,\delta)\)\emph{-metric LDP} with respect to \(d\) if, for all \(w,w'\in\mathcal{V}\) and all measurable \(S\subseteq\mathcal{Y}\),
\[
\Pr[\mathcal{M}(w)\in S] \le e^{\,\epsilon\, d(w,w')}\Pr[\mathcal{M}(w')\in S] + \delta.
\]
\end{definition}

\noindent Metric LDP is particularly well suited for text, as it allows stronger protection for pairs of semantically similar tokens (small \(d\)) while relaxing constraints for distant pairs.  
In the context of text, the metric \(d(\cdot,\cdot)\) can be instantiated using \emph{semantic or geometric similarity} between embeddings—such as cosine distance, Euclidean distance, or distances derived from contextualized representations \cite{chatzikokolakis2013broadening, chen2022customized}. 
This formulation captures the intuition that semantically similar tokens (e.g., ``\texttt{doctor}'' and ``\texttt{physician}'') should be harder to distinguish than unrelated ones (e.g., ``\texttt{doctor}'' and ``\texttt{banana}'').  
By embedding this structure directly into the privacy constraint, metric LDP aligns the notion of privacy with the geometry of language, providing a natural foundation for designing mechanisms that operate in embedding space.

    \begin{figure*}
        \centering
        \includegraphics[scale = 0.22]{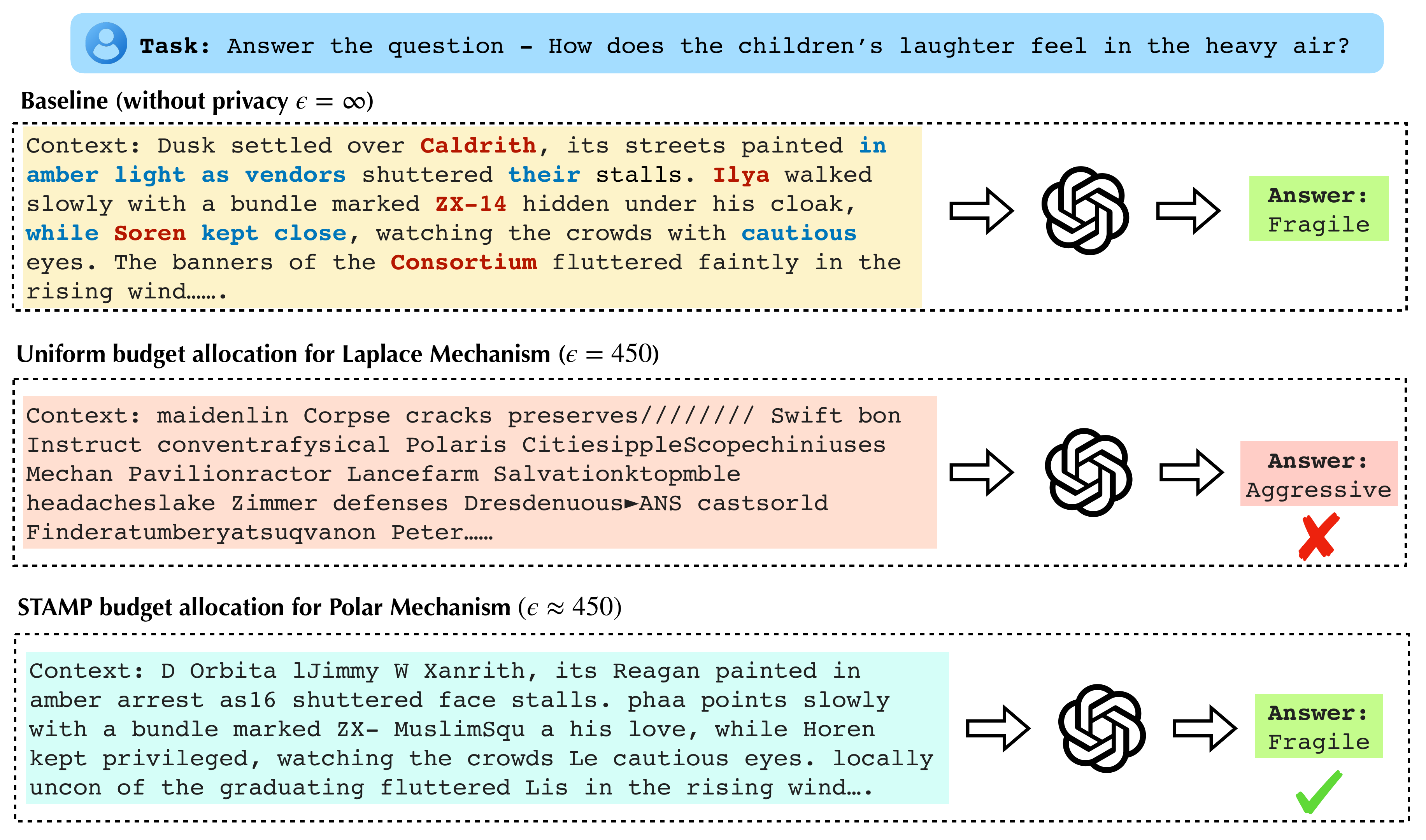}
        \caption{\textbf{Prompt-based question answering example }(in the baseline context, \textcolor{red}{\bf{red}} bold text indicates privacy-sensitive tokens; \textcolor{blue}{\bf{blue}} bold text denotes task-relevant text): 
        Uniform perturbation treats all tokens equally, wasting privacy on uninformative words and under-protecting sensitive but predictive ones. 
        \textsc{STAMP}  addresses this by stratifying tokens by privacy sensitivity and task relevance, and by aligning the privatization geometry with the decoding geometry.}
        \label{fig:example}
        \vspace{-5pt}
    \end{figure*}
    % This is fgure 2, moved here by force

\subsection{Related Work on Text Privacy}
\label{sec:motivation-related} 
    Previous studies have shown that some words directly reveal sensitive information—personal names, email addresses, company identifiers, locations—while others (e.g., stop words) pose little privacy risk and minimal task relevance \cite{li2023privacy, meisenbacher2024just}.
    Existing approaches for privatizing text often ignore this heterogeneity \cite{utpala2023locally, igamberdiev2023dp}: they either distort embedding spaces in ways that erase useful structure, or allocate privacy budgets based on language intrinsic properties without regard to task relevance \cite{meisenbacher2024spend}.
    \\
    \noindent \textbf{Isotropic embedding perturbation.} A common approach for perturbing embeddings involves adding isotropic Gaussian or Laplace noise \cite{igamberdiev2023dp}. While this leverages continuous representations, it implicitly assumes isotropy; yet empirical work shows strong \emph{anisotropy} in embedding spaces—some directions encode fine semantic distinctions while others capture frequency or style \cite{mu2017all, ethayarajh2019contextual}. 
    Uniform noise disregards this structure and can collapse distinctions embeddings are meant to preserve.
    
    \noindent \textbf{Directional privatization.} There exists methods on perturbing directions on the unit sphere, such as von Mises–Fisher (vMF)–based mechanisms and spherical-cap schemes, which offer DP or LDP guarantees under angular metrics \cite{weggenmann2021differential}. Directional noise has also been explored for privatizing gradient directions in deep learning models \cite{faustini2022directional}.
    In NLP, word-level metric-DP methods often inject noise in embedding space and then \emph{decode} by cosine nearest-neighbor back to tokens \cite{arnold2023driving,carvalho2023tem}.
    However, when directional noise is applied \emph{uniformly}—without task regard to the current task—and decoded in a mismatched geometry, it can still disrupt fine-grained relations needed downstream. 
    
    \noindent \textbf{Uniform budget allocation.}
    In prior works, a common simplification has been to allocate privacy budgets evenly across tokens \cite{carvalho2023tem}.
    This simplification ignores the fact that some tokens—such as named entities or specialized domain terms—are simultaneously highly sensitive and critical for model accuracy, while others contribute little to either privacy or utility \cite{feyisetan2020privacy, feyisetan2021private, xu2020differentially}.
    Recent work \cite{meisenbacher2024spend} redistributes the privacy budget $\epsilon$ using scores from intrinsic text properties (information content, POS tags, semantic deviation within context $C$), but remains largely \emph{task-agnostic}.

    % In contrast, our framework, \textsc{STAMP} , as described next explicitly models both a token’s \emph{privacy sensitivity} and its \emph{task relevance}, enabling selective and geometry-aligned privatization that better preserves downstream utility.
    % \textcolor{red}{Furthermore, while \textsc{STAMP} is instantiated here with the Polar mechanism to leverage geometric properties, it is a modular framework; other discrete or substitution-based mechanisms \cite{carvalho2023tem, feyisetan2020privacy} could in principle be integrated into the allocation scheme.}
    
    In contrast, STAMP addresses these limitations by explicitly modeling both a token’s privacy sensitivity and its task relevance. 
    Rather than treating privacy as a uniform property of the entire sequence, STAMP reflects the view that privacy is not merely an intrinsic property of text, but a contextual choice regarding which attributes must be concealed versus which are essential for utility. 
    Consequently, this enables a selective privatization that achieves a superior privacy-utility trade-off. 
    Furthermore, we rigorously distinguish the \textit{allocation strategy} from the \textit{perturbation mechanism}. 
    While STAMP is instantiated here with the Polar mechanism to leverage the directional geometry of embedding spaces, it is inherently a modular framework; other discrete or substitution-based mechanisms \citep{carvalho2023tem, feyisetan2020privacy} could in principle be integrated into the allocation scheme, provided they expose a tunable parameter for perturbation strength. 
    In this way, STAMP offers a generalizable principle for budget distribution that is orthogonal to the specific method of noise injection.
    
% This paragraph starts on top right of third page
\section{The \textsc{STAMP}  Framework}

    Unlike mechanisms that allocate a uniform privacy budget, \textsc{Stamp} partitions tokens along two binary axes: \emph{explicit sensitivity} (e.g., PII or NER labels) and \emph{task-dependent importance}.  
    While sensitivity can often be identified from metadata or entity type, importance depends on the specific downstream task---the same token or phrase may be crucial for one query yet irrelevant for another.  
    \textsc{Stamp} captures this by grouping tokens according to their similarity to a task or query representation.

    This task-aware selective design yields four token categories (illustrated in Table~\ref{tab:stamp_grid}) governed by two principles:  
    (i) sensitive tokens receive the strongest protection, and  
    (ii) among non-sensitive tokens, those with higher task importance are perturbed the least to preserve utility.  
    \textsc{Stamp} assigns each token $w$ to a group $c = g_T(w) \in \{1,2,3,4\}$ via a public, task-specific mapping $g_T$, and allocates privacy budgets at the \emph{group} level.  
    This selective allocation protects consequential attributes without erasing essential task information.  
    Formally, \textsc{Stamp} instantiates a task-aware metric-LDP mechanism in which budgets are assigned and perturbations applied within each group.
    
    \begin{definition}[Task-Aware Metric LDP]
    \label{def:task-cat-mldp}
        Let $T$ be a downstream task with a vocabulary partition 
        $\{\mathcal{V}_1^T, \mathcal{V}_2^T, \mathcal{V}_3^T, \mathcal{V}_4^T\}$.  
        A randomized mechanism $\mathcal{M}:\mathcal{V}\to\mathcal{Y}$ satisfies 
        \emph{task-aware metric LDP} with budget vector 
        $\boldsymbol{\epsilon}^T = (\epsilon_T^{(1)}, \epsilon_T^{(2)}, \epsilon_T^{(3)}, \epsilon_T^{(4)})$ 
        if, for every group $c \in \{1,2,3,4\}$, all tokens $w, w' \in \mathcal{V}_c^T$, and all measurable sets $S \subseteq \mathcal{Y}$,
        \[
            \Pr[\mathcal{M}(w)\in S] \le e^{\epsilon_T^{(c)} d_u(w,w')} \Pr[\mathcal{M}(w')\in S],
        \]
        where $d_u$ is a distance metric.
    \end{definition}
    
    \begin{theorem}[Privacy Guarantee of \textsc{Stamp}]
    \label{thm:task-mldp}
    Fix a task $T$ and assume that $T$ and its associated grouping map $g_T$ are public.  
    Then \textsc{Stamp} satisfies task-aware metric LDP with budget vector 
    $\boldsymbol{\epsilon}^T = (\epsilon_T^{(1)}, \epsilon_T^{(2)}, \epsilon_T^{(3)}, \epsilon_T^{(4)})$.
    \end{theorem}
    
    \noindent Grouping tokens and perturbing them with task-specific budgets yields $(\epsilon_T^{(c)},0)$-metric LDP for each group.  
    The complete proof appears in Appendix~\ref{app:proofs-task}.  
    In the limiting case where a token’s privacy budget $\epsilon = 0$ (for highly sensitive, low-utility tokens), the mechanism reveals no information about that token, effectively replacing it with a fully masked or placeholder symbol. 
    
    \medskip
    \noindent \textbf{Extending Token-Level Guarantees to Contexts.}  
    While the single-token guarantee secures individual embeddings, real applications release entire sentences or documents.  
    Privacy loss compounds when multiple tokens are observed together, so we extend the guarantee to sequences via standard composition in metric LDP.  
    Applying the mechanism independently to each position extends per-token guarantees to full contexts \cite{wang2020comprehensive, kairouz2015composition}.
    Consider a sequence $w^{(n)} = (w_1, \ldots, w_n)$ with group labels $c_i = g_T(w_i)$.  
    The \textsc{Stamp} mechanism applies independently to each token:
    \[
        \mathcal{M}_T^{(n)}(w^{(n)}) = \big(\mathcal{M}_T(w_1), \ldots, \mathcal{M}_T(w_n)\big).
    \]
    
    \begin{theorem}[Context-Level Privacy Guarantee of \textsc{Stamp}]
    \label{thm:composition}
    For any task $T$, $\mathcal{M}_T^{(n)}$ satisfies task-aware metric LDP with budget vector 
    $(\epsilon_{T,1}^{(c_1)}, \ldots, \epsilon_{T,n}^{(c_n)})$, where $\epsilon_{T,i}^{(c_i)} = \epsilon_T^{(c_i)}$.  
    In particular, for any pair of sequences $w^{(n)}, w'^{(n)}$ with the same group sequence $(c_1, \ldots, c_n)$ 
    and any measurable set $S$,
    \begin{align*}
            & \Pr[\mathcal{M}_T^{(n)}(w^{(n)})\in S] \\
            \le &
            e^{\sum_{i=1}^n \epsilon_T^{(c_i)} d_u(w_i,w'_i)}
            \Pr[\mathcal{M}_T^{(n)}(w'^{(n)})\in S].
        \end{align*}
    \end{theorem}
    
    \noindent Since \textsc{Stamp} privatizes each token independently, the sequence-level guarantee follows directly by composition; see Appendix~\ref{app:proofs-comp}.

 \subsection{Group Maps \& Privacy Budget Allocation}

    \noindent \textbf{Designing the Group Map.}  
    The grouping map $g_T$ partitions tokens along two dimensions: explicit sensitivity and task relevance.  
    Sensitivity can be detected using named-entity or PII recognizers (e.g., for names, organizations, or identifiers), 
    while task relevance can be estimated from gradient-based saliency, attention scores, or cosine similarity 
    to a task- or query-specific embedding.  
    Each token $w$ is thus assigned to one of four groups 
    based on its sensitivity and importance for task $T$.  
    Making $g_T$ public simplifies analysis and reproducibility, 
    as the grouping policy itself contains no private information 
    and enables well-defined, composable privacy guarantees.  
    
    \medskip
    \noindent \textbf{Allocating Privacy Budgets.}  
    Once groups are defined, privacy budgets $\boldsymbol{\epsilon}^T$ are assigned according to both sensitivity and relevance.  
    Sensitive tokens (e.g., names or identifiers) receive the strongest protection through smaller budgets, 
    while non-sensitive yet task-critical tokens are granted larger budgets to preserve downstream utility.  
    Tokens that are both non-sensitive and task-irrelevant tolerate stronger perturbations without degrading performance.  A practical rule of thumb is to assign privacy budgets that increase with task importance and decrease with privacy sensitivity. 
    In our experiments, we adopt the following proportional allocation:
    \[
    \epsilon_T^{(1)} : \epsilon_T^{(2)} : \epsilon_T^{(3)} : \epsilon_T^{(4)} = 2 : 1 : 4 : 3,
    \]
    where group 1 contains tokens that are both highly privacy-sensitive and task-critical (thus receiving a moderate budget to balance the trade-off); group 2 includes tokens that are highly sensitive but less important to the task (hence the smallest budget); group 3 corresponds to tokens with lower privacy sensitivity but high task importance (receiving the largest budget); and group 4 covers tokens with low sensitivity and moderate task importance. 
    This allocation provides a principled balance between protecting sensitive content and maintaining utility for downstream tasks.
    We adopt this ratio only as a transparent instantiation of the qualitative ordering: protect sensitive/task-unimportant tokens most and non-sensitive/task-important tokens least. 
    While not claimed to be theoretically optimal, this monotone profile demonstrates that structured allocation can yield improvements over uniform privacy budget allocation baselines.

\section{Polar Mechanism under Metric LDP}
\label{sec:polar}

    In this Section, we present a geometry-aware pipeline for privatizing token embeddings that aligns the \emph{perturbation geometry} with the \emph{decoding geometry}. 
    Rather than injecting isotropic noise in $\mathbb{R}^d$, we factor each embedding into \textbf{radial} and \textbf{angular} components and privatize each in its natural metric. For an embedding $\mathbf{e}\in\mathbb{R}^d$ of token $w$,
    \[
        \mathbf{e} = r\cdot \mathbf{u}, 
        \qquad r=\lVert \mathbf{e}\rVert_2, 
        \qquad \mathbf{u}=\frac{\mathbf{e}}{\lVert \mathbf{e}\rVert_2}\in\mathbb{S}^{d-1}.
    \]
    \begin{figure*}
        \centering
        \includegraphics[scale=0.24]{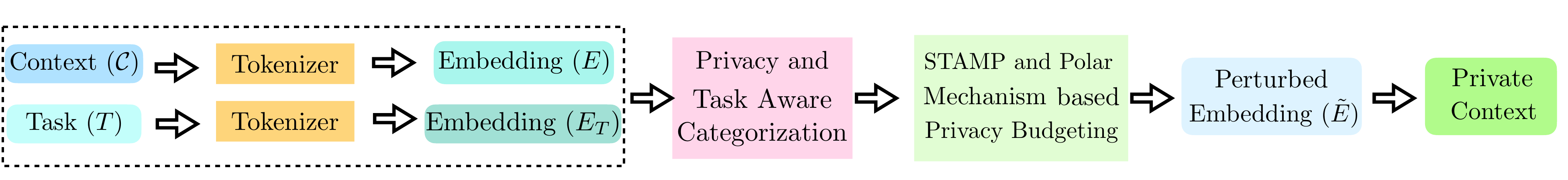}
        \caption{
        \textbf{\textsc{STAMP} privatization pipeline.} \
        Embeddings are decomposed into radial and angular components, perturbed under metric LDP, and decoded by angular proximity into privatized tokens suitable for downstream tasks.
        }
        \label{fig:pipeline}
    \end{figure*}
    \begin{figure}[t]
        \centering
        \includegraphics[scale = 0.3]{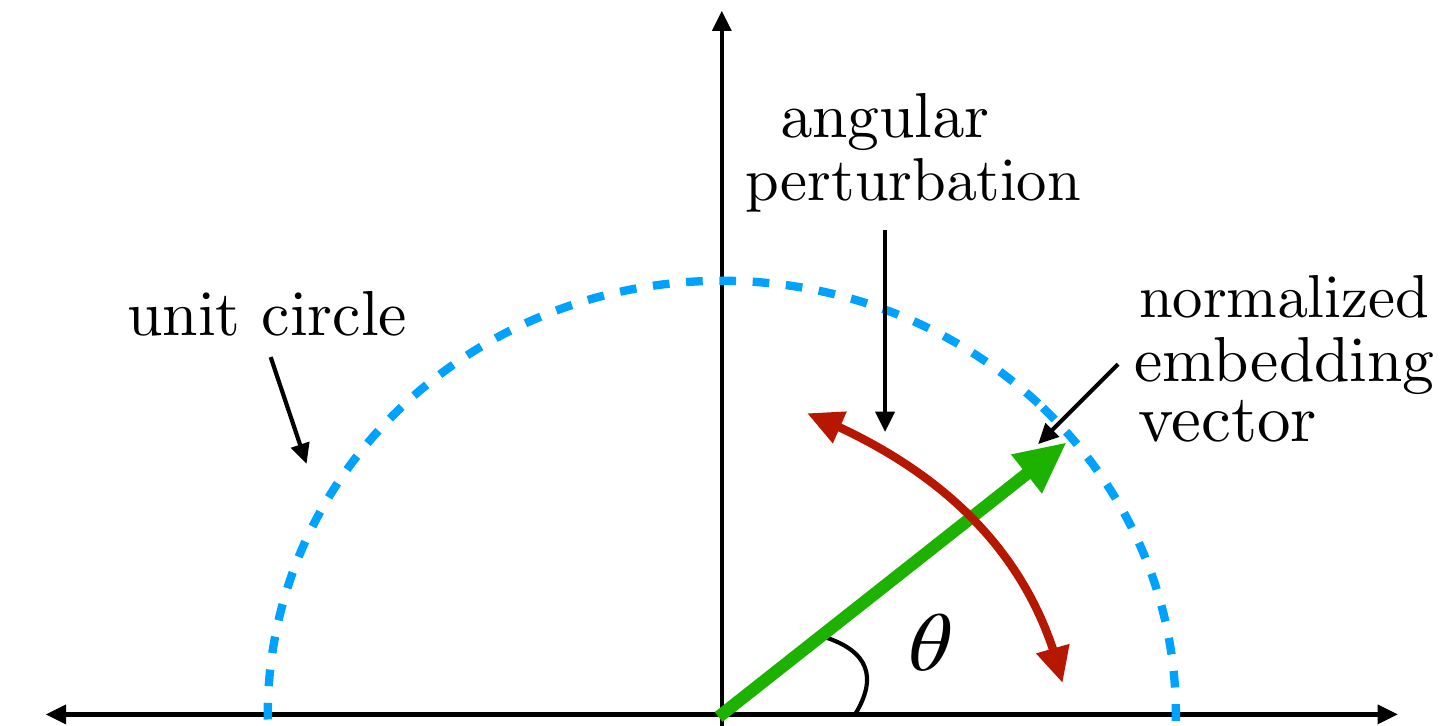}
        \caption{Directional privatization using the Polar mechanism, which perturbs the angle of the normalized embedding vector while normalizing its magnitude.}
        \label{fig:polar}
    \end{figure}
    \noindent \textit{(a) Magnitude/Radial Perturbation via Laplace}. We can privatize the radius with Laplace noise
    \[
        r' \sim \mathrm{Laplace}(r,b_r),\qquad b_r=\Delta_r/\epsilon_r,
    \]
    which gives $(\epsilon_r,0)$-metric LDP with respect to $d_r(r,r')=|r-r'|$.

    \noindent \textit{(b) Direction Perturbation via vMF noise}. 
    Independently, the direction is privatized on the unit sphere with a von Mises–Fisher (vMF) noise
    \[
        \mathbf{u}'\sim\mathrm{vMF}(\mu=\mathbf{u},\,\kappa),
    \]
    where $\kappa\!\ge\!0$ controls concentration. 
    Let $d_c(\mathbf{u},\mathbf{u}')=\lVert \mathbf{u}-\mathbf{u}'\rVert_2$ denote the chordal metric on $\mathbb{S}^{d-1}$. We formalize the privacy guarantee by the vMF noise based mechanism below.        
    \begin{theorem}[vMF mechanism satisfies metric LDP]
        \label{thm:vmf-mldp}
        Let $\mathbb{S}^{d-1}=\{u\in\mathbb{R}^d:\|u\|_2=1\}$ and consider the angular mechanism
        \begin{align*}
            &  \mathcal{M}(u)\sim\mathrm{vMF}(\mu=u,\kappa), \\
            &  \text{with density } 
            f(y\mid u)=C_d(\kappa)\exp\big(\kappa\,u^\top y\big), 
        \end{align*}
        where $C_d(\kappa)$ is the vMF normalizing constant.
        Then $\mathcal{M}$ satisfies \textit{metric local differential privacy} with respect to the metric
        $d(u,u')=\|u-u'\|_2$, in the sense that for all $u,u'\in\mathbb{S}^{d-1}$ and all measurable $S\subseteq \mathbb{S}^{d-1}$,
        \begin{align*}
            \Pr[\mathcal{M}(u)\in S]\le e^{\big(\kappa\,\|u-u'\|_2\big)} \Pr[\mathcal{M}(u')\in S].
        \end{align*}
        Equivalently, the mechanism is $\epsilon$-metric LDP with parameter $\epsilon=\kappa$ under metric $d(u,u')=\|u-u'\|_2$.
    \end{theorem}
    \noindent The full proof is deferred to Appendix~\ref{app:proofs-vmf}. Combining the two independent perturbations yields a \emph{two-budget} guarantee $(\epsilon_r,\epsilon_u)$.

    \noindent \textbf{Semantic Decoding}
    The polar mechanism outputs a privatized embedding $\mathbf{e}' = r'\,\mathbf{u}' \in \mathbb{R}^d$ which may not correspond to a valid token. Therefore, a downstream \textit{decoding step} identifies the token whose embedding is most similar in cosine space:  
    \[
        \mathrm{Decode}(\mathbf{e}')=\arg\max_{\mathbf{v}\in E}\frac{{\mathbf{e}'}^\top \mathbf{v}}{\lVert \mathbf{e}'\rVert_2\,\lVert \mathbf{v}\rVert_2},
    \]
    Because the above decoding rule depends only on direction, perturbations of 
$\|\mathbf{e}'\|_2$ cannot affect the argmax---a property we call \emph{radial invariance}. This motivates discarding the radial (magnitude) perturbation entirely and privatizing only the directional component.

\noindent \textbf{Normalized Polar Mechanism.}  
Motivated by this radial invariance, we project each embedding onto the unit sphere—discarding the original radius entirely—and privatize only its direction:
\[
\bar{\mathbf{e}} = \frac{\mathbf{e}}{\|\mathbf{e}\|_2}, 
\qquad 
\mathbf{u}' \sim \mathcal{M}_u(\bar{\mathbf{e}})\ \text{(vMF)}.
\]
This yields $(0, \epsilon_u)$-metric LDP under the angular metric and reduces Polar to a single-budget, direction-only mechanism in practice.  
In particular, the \emph{magnitude is perfectly private}: the released vector has fixed norm ($\|\mathbf{u}'\|_2 = 1$), so the output distribution is independent of $r$, i.e., $0$-metric LDP in the radial component.  
As illustrated in Figure~\ref{fig:polar}, the normalized Polar mechanism perturbs only the angular component while keeping the magnitude constant.  

\medskip
\noindent \textbf{Computational Complexity}. Decoding then reduces to nearest-neighbor search over random directions.  
The computational cost is that of standard nearest-neighbor search on $\mathbb{S}^{d-1}$: $\mathcal{O}(|\mathcal{V}|\,d)$ for exact search, and sublinear in $|\mathcal{V}|$ using modern approximate nearest-neighbor (ANN) methods such as HNSW or FAISS \cite{johnson2019billion, malkov2018efficient}.  
Formally, on the unit sphere, ranking by cosine similarity is equivalent to nearest-neighbor search under both geodesic (angular) and chordal (Euclidean) distances, aligning the decoding geometry with our privacy metric. We formalize this property in the next Proposition.

\begin{table*}[t]
\centering
\scriptsize
\setlength{\tabcolsep}{3.5pt}
\renewcommand\arraystretch{1.15}
% Two-line cell helper (mean on top, std below as ± value)
\newcommand{\twoline}[2]{\shortstack{$#1$\\ \,#2}}

\begin{tabular}{l *{9}{c}}
\toprule
% first column header spans the two header rows as well
\multirow{2}{*}{\shortstack[l]{\textbf{Group Privacy Budgets}\\ \small $\{G_1,G_2,G_3,G_4\}$}}
& \multicolumn{2}{c}{\textbf{SQuAD (Cosine similarity)}} & \multirow{2}{*}{\shortstack{\textbf{Non-Private}\\ \textbf{Baseline}}}
& \multicolumn{2}{c}{\textbf{Yelp (Acc \%)}}             & \multirow{2}{*}{\shortstack{\textbf{Non-Private}\\ \textbf{Baseline}}}
& \multicolumn{2}{c}{\textbf{AG News (Acc \%)}}           & \multirow{2}{*}{\shortstack{\textbf{Non-Private}\\ \textbf{Baseline}}} \\
\cmidrule(lr){2-3}\cmidrule(lr){5-6}\cmidrule(lr){8-9}
& \textbf{Polar (vMF)} & \textbf{Laplace} &
& \textbf{Polar (vMF)} & \textbf{Laplace} &
& \textbf{Polar (vMF)} & \textbf{Laplace} & \\
\midrule
{\{150, 50, 450, 350\}}     & 0.393  & 0.325   & 0.839
                            & 0.360  & 0.220   & 0.580
                            & 0.540  & 0.560   & 0.920 \\
{\{200, 100, 500, 400\}}    & 0.470  & 0.334   & 0.839
                            & 0.380  & 0.140   & 0.580
                            & 0.640  & 0.580   & 0.920 \\
{\{250, 150, 550, 450\}}    & 0.587  & 0.335   & 0.839
                            & 0.480  & 0.180   & 0.580
                            & 0.680  & 0.560   & 0.920 \\
{\{300, 200, 600, 500\}}    & 0.654  & 0.341   & 0.839
                            & 0.560  & 0.200   & 0.580
                            & 0.760  & 0.520   & 0.920  \\
{\{350, 250, 650, 550\}}    & 0.833  & 0.343   & 0.839
                            & 0.560  & 0.140   & 0.580
                            & 0.800  & 0.520   & 0.920 \\
\bottomrule
\end{tabular}

 \caption{Polar vs. Laplace vs. Baseline at matched protocol (cosine decoding for all), embedding dimension at $d{=}768$. 
        Rows are privacy points; columns group mechanisms by dataset. 
        \textit{SQuAD (Cosine)} cells show cosine similarity. 
        \textit{Yelp/AG News} cells show accuracy.}
\label{tab:polar_laplace_baseline_7eps_transposed}
\end{table*}

\begin{proposition}[Equivalence of Semantic Decoding and Nearest-Neighbor Search on the Sphere]
\label{thm:cosine-nn-equivalence}
Let $E = \{v_1, \ldots, v_{|\mathcal{V}|}\} \subset \mathbb{R}^d$ denote the set of token embeddings, and let 
$\hat{e} = e' / \|e'\|_2$ be the direction of a privatized vector $e' \neq 0$.  
Then the following decoding rules are equivalent:
\begin{align}
\mathrm{Decode}(e') 
&= \arg\max_{v \in E} \frac{e'^\top v}{\|e'\|_2 \, \|v\|_2} \\
&= \arg\max_{v \in E} \hat{e}^\top \hat{v} \label{eq:cos} \\
&= \arg\min_{v \in E} \arccos\!\big(\hat{e}^\top \hat{v}\big) \label{eq:geo} \\
&= \arg\min_{v \in E} \|\hat{e} - \hat{v}\|_2, \label{eq:euclid}
\end{align}
where $\hat{v} = v / \|v\|_2$ for $v \neq 0$, and any $v$ with $\|v\|_2 = 0$ can be excluded.  
In particular, decoding by cosine similarity is identical to nearest-neighbor search on the unit sphere under both geodesic (angular) and Euclidean (chordal) distances.
\end{proposition}

\noindent The proof is provided in Appendix~\ref{app:proofs-ml}.  
Angular decoding preserves semantic neighborhoods while discarding norm information, which often correlates with token frequency or salience.  
Together with Polar perturbation, this alignment yields privatized text consistent with cosine semantics without introducing additional computational overhead.

    \begin{algorithm}[h]
        \caption{STAMP--Polar Mechanism}
        \label{alg:polar}
        \begin{algorithmic}[1]
            \REQUIRE Context $c=(w_1,\ldots,w_n)$; task $T$; grouping map $g_T:\mathcal{V}\!\to\!\{1,2,3,4\}$; group privacy budgets $\{\epsilon_u^{(c)}\}_{c=1}^4$; vocabulary embeddings $V=\{\mathbf{v}\}$
            \ENSURE Privatized sequence $\tilde c=(w'_1,\ldots,w'_n)$
            \FOR{$i=1$ \TO $n$}
                \STATE $c_i \gets g_T(w_i)$
                \STATE $\mathbf{e}_i \gets \mathbf{e}(w_i)$;\quad $\bar{\mathbf{e}}_i \gets \mathbf{e}_i / \|\mathbf{e}_i\|_2$
                \STATE $\kappa_i \gets \epsilon_u^{(c_i)}$
                \STATE Sample $\mathbf{u}'_i \sim \mathrm{vMF}\!\left(\mu=\bar{\mathbf{e}}_i,\,\kappa_i\right)$
                \STATE $w'_i \gets \arg\max_{\mathbf{v}\in V} \dfrac{{\mathbf{u}'_i}^\top \mathbf{v}}{\|\mathbf{v}\|_2}$ \textit{(decoding)}
            \ENDFOR
            \STATE \textbf{return} $\tilde c=(w'_1,\ldots,w'_n)$
        \end{algorithmic}
    \end{algorithm}

    \vspace{5pt}
    \noindent Figure~\ref{fig:pipeline} illustrates the overall \textsc{STAMP} workflow: tokens are grouped using the public task map $g_T$, group-level privacy budgets are assigned, embeddings are privatized, and the resulting representations are decoded back to tokens.
    The privatization step is instantiated with the \emph{normalized Polar} mechanism.
    Algorithm~\ref{alg:polar} summarizes the context-level procedure: for each token position $i$, determine its group label $c_i = g_T(w_i)$, normalize its embedding, sample a vMF-perturbed direction with concentration parameter $\kappa_i = \epsilon_u^{(c_i)}$, and decode by cosine similarity to obtain the privatized token $w'_i$.
    
    \noindent \textbf{Span-level Constraints.}
    While STAMP perturbs tokens individually, the budgeting mechanism operates at the span level. 
    If a detector identifies a multi-token entity (e.g., 'New York'), the entire span is assigned to the same privacy group. 
    While boundary detection errors are inevitable, STAMP is robust to mild inconsistencies because the Polar mechanism perturbs only direction, maintaining local semantic neighborhoods even if adjacent tokens receive slightly different budgets.

    % Figure~\ref{fig:pipeline} presents the general \textsc{STAMP} pipeline: tokens are grouped by the public task map $g_T$, group-level privacy budgets are assigned, embeddings are privatized, and outputs are decoded to tokens. 
    % Moreover, we instantiate the privatization step with the \emph{normalized Polar} mechanism. Algorithm~\ref{alg:polar} is the context-level master procedure: for each position $i$, assign $c_i=g_T(w_i)$, normalize the embedding, sample a vMF-perturbed direction with $\kappa_i=\epsilon_u^{(c_i)}$, and decode by cosine similarity to produce $w'_i$.

\section{Experiments}
\label{sec:experiments} In this Section, we present the experimental results on three commonly used datasets: SQuAD \cite{rajpurkar2016squad}, Yelp \cite{yelp_open_dataset}, and AG News \cite{zhang2015charcnn}, and analyzed the privacy utility trade-off for Polar (vMF), Laplace Mechanism and non-private baseline Mask and Fill. Specifically, we tried answering two questions:
(i) at matched privacy budget, how do the Polar (angular vMF) and isotropic Laplace mechanisms compare in utility; and
(ii) at matched privacy budget, how does \textsc{STAMP} framework compare to uniform framework.
    %We evaluate on three commonly used datasets: SQuAD, Yelp \cite{yelp_open_dataset}, and AG News \cite{zhang2015charcnn}.

    \begin{figure}[t]
        \centering
        \includegraphics[scale = 0.42]{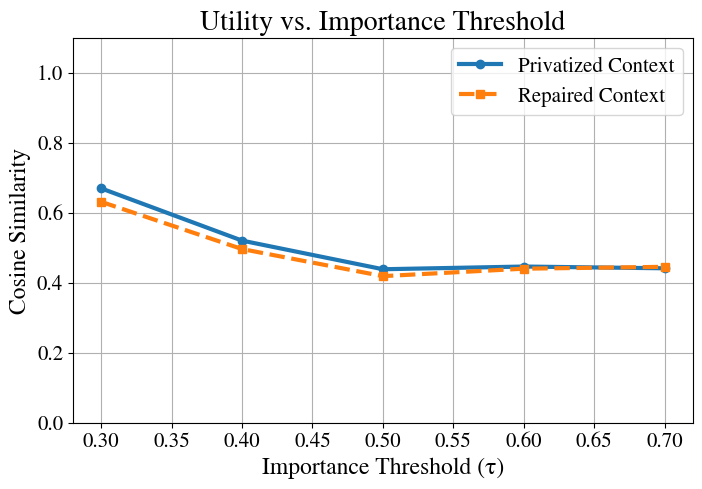}
        \caption{
        The parameter $\tau$ serves as the critical decision boundary that classifies tokens as task-relevant or irrelevant based on their cosine similarity to the task representation. 
        A sweep of $\tau$ reveals that downstream utility (measured by cosine similarity between original and privatized contexts) stabilizes around $\tau=0.5$, indicating a robust operating point that preserves task-critical signals without over-perturbing the context.}
        \label{fig:tau}
    \end{figure}

    % Main results, 2 row of 3 figures side by side
    \begin{figure*}[t]
    \includegraphics[scale = 0.24]{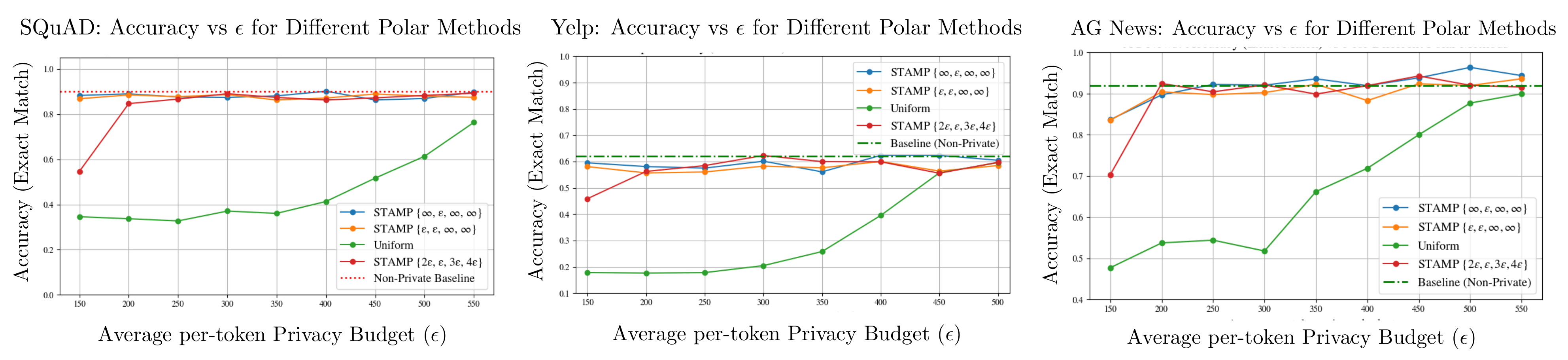}
        % -------- Bottom row: \textsc{STAMP}  vs UNIFORM (mid-range \bar{\epsilon}) --------
        %\begin{subfigure}{0.32\textwidth}
          %\includegraphics[width=\linewidth]{figures/STAMP_VS_Uniform_Fantasy.png}
          %\caption{SQuAD}
       % \end{subfigure}
        %\begin{subfigure}{0.32\textwidth}
          %\includegraphics[width=\linewidth]{figures/STAMP_VS_Uniform_Yelp.png}
          %\caption{Yelp}
        %\end{subfigure}
        %\begin{subfigure}{0.32\textwidth}
          %\includegraphics[width=\linewidth]{figures/STAMP_VS_Uniform_AGnews.png}
         % \caption{AG News}
        %\end{subfigure}
        
        \caption{\textsc{STAMP} vs. Uniform (single angular budget) under the Polar mechanism across three tasks: SQuAD (QA), Yelp (sentiment), and AG News (topic classification). 
        Curves show task performance versus the \emph{base} angular budget $\epsilon_0$; both frameworks use normalized polar mechanism and cosine decoding. \textsc{STAMP} variants (e.g., private-only, private-low, heuristic) retain higher utility by concentrating noise on sensitive, low-importance spans, while Uniform converges as noise vanishes at large $\epsilon_0$.}
        \label{fig:comp-grid}
        \end{figure*}
                
        \noindent \textbf{Evaluators, rewrite, and metrics.}
        For QA we use GPT-4 prompted to answer questions from the privatized context; for classification we use existing models \cite{kmack_yelp_review_classifier} \cite{textattack_distilbert_ag_news} trained on the original (non-privatized) training splits of Yelp and AG News, with privatization applied only at inference on test. Unless stated otherwise, results are reported without rewriting. For SQuAD, we additionally include an optional post-privatization rewrite variant that performs coherence repair without introducing external facts.
        QA metrics are EM (exact match) accuracy and cosine similarity between predicted and gold answers (Sentence-BERT); classification tasks use accuracy.
        We report dataset means averaged over three runs.
        Additional experiment details are provided in Appendix \ref{app:exp}.  
        Our code is available at \cite{stamp_code}.

    \subsection{Experimental Setup} We evaluate our proposed approach on three datasets for different tasks—SQuAD (question answering), Yelp (sentiment rating), and AG News (topic classification)—using task-specific metrics: answer cosine similarity for QA and accuracy for classification. All comparisons are conducted at matched per-token ${\epsilon}$ values.
    % Tokens are partitioned into a $2{\times}2$ grid by importance $\times$ privacy: importance is $\cos(\text{token},\text{query}) \ge \tau$ with a fixed threshold $\tau$, and privacy follows standard NER/PII rules. 

    Tokens are partitioned into a $2{\times}2$ grid by importance $\times$ privacy as shown in Figure \ref{tab:stamp_grid} (b). 
    Task importance is determined by the cosine similarity between a token's embedding and the task/query representation, defined as $\cos(\text{token},\text{query}) \ge \tau$. 
    To select this threshold, we performed a parameter sweep for $\tau \in \{0.3, 0.4, 0.5, 0.6, 0.7\}$, measuring the cosine similarity between the original and privatized contexts.
    Figure \ref{fig:tau} shows that utility stabilizes around $\tau=0.5$; we therefore fixed $\tau=0.5$ as the representative default for all main experiments.
    Privacy sensitivity follows standard NER/PII rules, treating tokens tagged as Person, Location, Organization, or Numeric Identifiers as sensitive. 
    Comparisons are conducted at 'matched privacy budget,' meaning both mechanisms satisfy $\epsilon$-metric-LDP for the same value of $\epsilon$. 
    Any utility difference is therefore attributable to the mechanism's geometry (directional vs. isotropic) or the allocation strategy, rather than an increased privacy budget.
    %We report realized $(\#G_1,\ldots,\#G_4)$ and the resulting $\bar{\epsilon}$.}

    % Mechanism-wise, we compare Polar—the normalized, direction-only vMF channel on the unit sphere—to isotropic Laplace noise in $\mathbb{R}^d$, holding decoding fixed: both methods decode by cosine nearest neighbor (implemented as Euclidean NN over a unit-normalized embedding table). Unless noted, embeddings are GPT-2 with $d{=}768$; under our metric-LDP formulation, maintaining comparable angular indistinguishability suggests scaling $\epsilon$ roughly linearly in $d$, but all runs here use the same $d$. For SQuAD, a capable LLM answers questions from privatized contexts; for Yelp and AG News, off-the-shelf classifiers are trained on original (non-privatized) 

    \subsection{Polar vs Laplace}
    We compare \emph{Polar mechanism} (normalized, direction-only vMF) to \emph{isotropic Laplace mechanism} under identical privacy budgets and the same embedding dimension ($d{=}768$). 
    For this comparison only, we sweep a \emph{per-token} privacy budget specific to the mechanism (directional $\epsilon_{\text{dir}}$ for Polar, isotropic $\epsilon_{\text{iso}}$ for Laplace) and report the mean over the sweep. Table~\ref{tab:polar_laplace_baseline_7eps_transposed} summarizes results (SQuAD: answer cosine similarity; Yelp/AG News: accuracy), averaged over three seeds.
    Across datasets, Polar consistently yields a stronger utility–privacy trade-off: at relatively low budgets, Laplace collapses toward chance-level performance while Polar remains effective; as the budget increases, Polar mechanism's performance improves smoothly and approaches the non-private baseline, whereas Laplace continues to lag across the sweep.

    \subsection{STAMP vs Uniform}
    At matched per-token privacy budget $\epsilon$ (Fig.~\ref{fig:comp-grid}), \textsc{STAMP} consistently outperforms a Uniform scheme in utility, especially in the low–to–mid privacy regime. 
    \textsc{STAMP} preserves high-importance content while concentrating noise on spans that are privacy-sensitive yet task-unimportant. 
    Utility is highest when perturbations are confined to such a small subset; as the protected set expands (from one group to two, three, and then all four groups), performance degrades monotonically. 
    Uniform—allocating the same budget to every token—forms the lower envelope across the grid.
    
    \noindent Although \textsc{STAMP} perturbs only a subset of groups, it effectively protects privacy-sensitive content while leaving high-importance tokens largely intact, thereby avoiding the typical utility penalty. 
    Moreover, question-conditioned importance (e.g., SQuAD) amplifies \textsc{STAMP}'s gains, whereas fixed prompts (Yelp/AG News) dampen them by stabilizing group proportions.
    Beyond utility, \textsc{STAMP} provides a controllable knob over which content to protect, aligning with the paper’s claim that privacy preferences are inherently subjective and task-dependent.

    \subsection{Computational Overhead}
    We analyze the overhead of STAMP relative to standard isotropic mechanisms (e.g., Laplace).
    \paragraph{Grouping and Budgeting.} 
    The task-aware grouping step involves NER/PII detection and computing token-task similarities. 
    Its complexity is linear in the number of tokens ($O(n)$). 
    In practice, this cost is minimal because it relies on the same embedding encoder forward pass required by the downstream model itself. 
    In our experiments on SQuAD, this partitioning step added only \textbf{0.002s} per example (avg. 180 tokens).
    \paragraph{Mechanism Comparison.}
    The Polar mechanism requires sampling from the von Mises-Fisher (vMF) distribution. 
    While geometrically distinct from isotropic noise, efficient rejection sampling schemes allow vMF sampling to scale linearly with embedding dimension ($O(d)$), comparable to the efficiency of Gaussian or Laplace sampling.
    \paragraph{Empirical Latency.}
    To quantify the actual overhead, we measured the total wall-clock time on SQuAD validation examples (avg. 180 tokens). 
    The average per-example runtime was \textbf{35.16s} (\textbf{195 ms/token}) for STAMP-Polar compared to \textbf{34.54s} (\textbf{192 ms/token}) for the Laplace baseline. 
    This comparison confirms that STAMP with Polar mechanism operate with essentially the same computational latency as uniform isotropic baselines.

\section{Conclusion}

    We introduced \textsc{STAMP} , a task-aware mechanism for privatizing text under local differential privacy.
    By combining geometry-aware perturbation with task-dependent budget allocation, \textsc{STAMP}  aims to balance privacy and utility across diverse NLP tasks.
    We provide formal guarantees and report empirical results on SQuAD, Yelp, and AG News. 
    STAMP outperforms Uniform at matched per-token budgets and Polar surpasses isotropic Laplace.
    Future work will address dynamic tasks and sequence-level dependencies, moving toward more robust privacy-preserving NLP systems.

%% This should be sharp end of the page 8
\newpage
\section{Limitations}

    \textsc{STAMP} represents an initial step toward task-aware LDP for text, but it carries several limitations.  
    Most notably, the framework assumes the availability of a meaningful task description that can be encoded as a fixed representation at inference time; this assumption may not hold in interactive, open-ended, or multi-turn scenarios. 
    Additionally, token-level relevance grouping is based on static embedding similarity, which may miss nuances of functional importance for tasks involving complex syntactic structures, discourse-level reasoning, or external knowledge. 
    Moreover, locally privatizing high-dimensional embeddings ($d=768$) requires relatively large per-token $\epsilon$ values to retain utility, which is a known challenge in text LDP. 
    Our contribution is relative: under a fixed privacy budget, \textsc{STAMP}  reallocates protection to where it is most needed.
    Also, this framework relies on the quality of the task/sensitivity oracle; while robust to mild errors, imperfect PII boundary detection remains a limitation inherent to entity-based approaches.
    Finally, \textsc{STAMP}  does not currently account for long-range dependencies or structured interactions between tokens during budget allocation, which may result in over- or under-perturbation in semantically rich contexts. 
    We leave addressing these challenges to future work.

\section{Ethical Considerations}

    \textsc{STAMP} is designed to enhance privacy in NLP by providing formal LDP guarantees, enabling the use of sensitive corpora such as email or financial text without exposing raw content. 
    This represents a positive step toward building systems where privacy is a first-class objective alongside accuracy. 
    
    \noindent At the same time, several risks remain. 
    Task-aware budgets may protect some attributes more strongly than others, raising the possibility of uneven coverage across demographic or domain-specific categories. 
    Careless or adversarial configuration of privacy budgets could also weaken effective guarantees, giving a false sense of protection. 
    Finally, \textsc{STAMP}  focuses on token-level privatization and does not address broader concerns such as fairness, data misuse, or downstream harms that can arise even from privatized text.  
    
    \noindent We emphasize that \textsc{STAMP}  is a methodological tool, not a complete solution, and that responsible deployment requires auditing, fairness checks, and clear communication of privacy parameters.

\newpage 
\section*{Acknowledgement}

    This work was supported by NSF grants CCF 2100013, CNS 2209951, CNS 2317192, and U.S. Department of Energy, Office of Science, Office of Advanced Scientific Computing under Award Number DE-SC-ERKJ422, and by NIH through Award 1R01CA261457-01A1.
    
    \noindent Notice: This manuscript has been authored by UT-Battelle, LLC, under contract DE-AC05-00OR22725 with the US Department of Energy (DOE). 
    The US government retains and the publisher, by accepting the article for publication, acknowledges that the US government retains a nonexclusive, paid-up, irrevocable, worldwide license to publish or reproduce the published form of this manuscript, or allow others to do so, for US government purposes. 
    DOE will provide public access to these results of federally sponsored research in accordance with the DOE Public Access Plan (\url{https://www.energy.gov/doe-public-access-plan}).

\bibliography{STAMP_CR}
%%%%%%%%%%%%%%%%%%%%%%%%%%%%%%%%%%%%%%%%%%%%%%%%%%%%%%%%%%%%
\newpage
\appendix

% ----- Appendix index (unnumbered) -----
\section*{Appendix Index}
This appendix provides supplementary material referenced in the main text:
\begin{itemize}
  \item Threat Model (\S\ref{app:threat})
  \item Privacy Guarantees: Detailed Proofs (\S\ref{app:proofs})
  \begin{itemize}
     \item Proof of Theorem~1
     \item Proof of Theorem~2 (Sequence-level composition)
     \item Proof of Theorem~3 (vMF is $\kappa$-mLDP)
     \item Proof of Proposition~1 (Cosine decoding is MLE under vMF)
  \end{itemize}
  \item Additional Experimental Results (\S\ref{app:exp})
\end{itemize}

% ===== First appendix section =====
\section{Threat Model}
\label{app:threat}

    \paragraph{Setting.}
    Each user holds a private text input $C$. Before transmission or use for task $T$, the user applies the local mechanism $M$ to obtain a perturbed output $\tilde{C}=M(C,T)$ and only $\tilde{C}$ is shared.
    
    \paragraph{Aggregator model (honest-but-curious).}
    The central server follows the protocol (collects $\tilde{C}$) but may attempt to infer information about the original inputs $C$ from $\tilde{C}$.
    
    \paragraph{Adversary’s goals.}
    Infer sensitive content or attributes of $C$, including:
    \begin{enumerate}
      \item identifying specific sensitive tokens/spans (e.g., names, locations, medical terms);
      \item inferring user attributes correlated with $C$ (e.g., demographics or preferences);
      \item partial or full reconstruction of $C$.
    \end{enumerate}
    
    \paragraph{Adversary’s knowledge.}
    We assume knowledge of: 
        (i) the mechanism $M$ and its parameters,
        (ii) the vocabulary/embeddings, 
        (iii) the public grouping map $g_T$ for task $T$, 
        and (iv) any auxiliary background knowledge. 
    The adversary observes only privatized outputs (and decoded tokens, if decoding is used), not raw inputs.
    
    \paragraph{User/device assumptions.}
    Randomization occurs locally; per-position channels are sampled independently conditioned on inputs and groups; the RNG is not adversarially controlled.
    
    \paragraph{Public disclosure.}
    The grouping map $g_T$ and mechanism hyper-parameters (e.g., $\kappa$ or $\epsilon$ per group) are treated as public; privacy does not rely on secrecy of $g_T$.
    
    \paragraph{Out-of-scope threats.}
    We do not address side channels (timing, memory), compromised clients, corrupted RNG, or post-aggregation leakage unrelated to the local mechanism.
    
    \paragraph{Guarantees in scope.}
    Protection follows the metric LDP guarantees for the angular channel (Theorem~\ref{thm:vmf-mldp}) and their sequence-level composition (Theorem~\ref{thm:composition}). 
    Experiments report the \emph{per-token} budget $\bar{\epsilon}$ and realized group counts $(\#G_1,\#G_2,\#G_3,\#G_4)$; $\bar{\epsilon}$ is computed from the observed mix for comparability across runs.

\section{Privacy Guarantees: Detailed Proofs}
\label{app:proofs}

    % Minimal notation so this section is self-contained
    \paragraph{Notation.}
    We work on the unit sphere $\mathbb{S}^{d-1}=\{y\in\mathbb{R}^d:\lVert y\rVert_2=1\}$.
    Geodesic (angular) distance is $d_g(u,v)=\arccos(u^\top v)$ and chordal distance is $d_2(u,v)=\lVert u-v\rVert_2$.
    They satisfy $d_2(u,v)=2\sin\!\big(\tfrac{1}{2}d_g(u,v)\big)\le d_g(u,v)$.
    The vMF density with mean $\mu\in\mathbb{S}^{d-1}$ and concentration $\kappa\ge 0$ is
    $f(y\mid \mu)=C_d(\kappa)\exp\!\big(\kappa\,\mu^\top y\big)$, where $C_d(\kappa)$ is independent of $\mu$.
    
    \subsection{Proof of Theorem~\ref{thm:task-mldp} (MLDP for \textsc{STAMP} )}
    \label{app:proofs-task}
        \begin{proof}
            Let $g_T$ be the public grouping map for task $T$ assigning token $w$ to group $c=g_T(w)$.
            Condition on any fixed group $c$.
            The angular channel is $Y\sim \mathrm{vMF}(\mu(w),\kappa_c)$; by Theorem~\ref{thm:vmf-mldp}, it satisfies $(\epsilon,0)$-metric LDP with $\epsilon=\kappa_c$ on $\mathbb{S}^{d-1}$.
            As $g_T$ is a deterministic (public) function of the input, restricting to the subset with $g_T(w)=c$ preserves the worst-case guarantee.
            Therefore the mechanism is $(\epsilon_T^{(c)},0)$-metric LDP with $\epsilon_T^{(c)}=\kappa_c$ for each group $c$.
        \end{proof}
    
    \subsection{Proof of Theorem~\ref{thm:composition} (Sequence-level composition)}
    \label{app:proofs-comp}
        \begin{proof}
            Let $x=(x_1,\dots,x_n)$ and independent per-position channels $M_t$ that are $\epsilon_t$-mLDP w.r.t. a single-token metric $d$.
            For any measurable $S\subseteq \mathcal{Y}^n$,
            \begin{align*}
                \frac{\Pr[M(x)\in S]}{\Pr[M(x')\in S]}
                & \le\ \prod_{t=1}^{n}\exp\!\big(\epsilon_t\, d(x_t,x_t')\big) \\
                & =\ \exp\!\Big(\sum_{t=1}^{n}\epsilon_t\, d(x_t,x_t')\Big).       
            \end{align*}
        
            Thus the product channel is $(\sum_t \epsilon_t)$-mLDP w.r.t. the sequence metric $d_\Sigma(x,x')=\sum_t d(x_t,x_t')$.
            For group-specific budgets, set $\epsilon_t=\epsilon(c_t)$ with $c_t=g_T(x_t)$ to obtain the stated form.
        \end{proof}

    \subsection{Proof of Theorem~\ref{thm:vmf-mldp} (vMF is $\kappa$-mLDP)}
    \label{app:proofs-vmf}
        \begin{proof}
            Fix unit $\mu,\nu\in\mathbb{S}^{d-1}$ and measurable $S\subseteq\mathbb{S}^{d-1}$.
            Using the vMF density and cancellation of $C_d(\kappa)$,
            \[
            \begin{aligned}
                \log\frac{f(y\mid \mu)}{f(y\mid \nu)}
                &= \kappa\,(\mu-\nu)^\top y \\
                &\le \kappa\,\|\mu-\nu\|_2\,\|y\|_2 \\
                &= \kappa\,\|\mu-\nu\|_2
            \end{aligned}
            \]
            By Cauchy–Schwarz and $\lVert y\rVert_2=1$, exponentiating and taking supremum over $S$ yields
            \[
                \Pr[M(\mu)\in S]\ \le\ \exp\!\big(\kappa\lVert \mu-\nu\rVert_2\big)\ \Pr[M(\nu)\in S],
            \]
            i.e., $(\epsilon,0)$-metric LDP with $\epsilon=\kappa$ under $d_2$.
            Since $d_2(\mu,\nu)\le d_g(\mu,\nu)$ the same bound holds under $d_g$.

            \noindent Let $\theta=d_g(u,v)\in[0,\pi]$. 
            Then $d_2(u,v)=2\sin(\theta/2)\le \theta=d_g(u,v)$ because $\sin x\le x$ for $x\ge 0$.
        \end{proof}
        
    \subsection{Proof of Proposition~\ref{thm:cosine-nn-equivalence} (Cosine decoding)}
    \label{app:proofs-ml}
        \begin{proof}
            For fixed $\kappa$, $\log f(y\mid \mu(w))=\kappa\,\mu(w)^\top y+\mathrm{const}$, hence
            \[
                \arg\max_{w}\log f(y\mid \mu(w))=\arg\max_{w}\mu(w)^\top y,
            \]
            which is exactly cosine nearest-neighbor on $\mathbb{S}^{d-1}$.
        \end{proof}

\section{Additional Experimental Results}
\label{app:exp}

    \paragraph{Prompts for Rewriting (Coherent Repair).}
    For the optional Coherence Repair step, we utilized \texttt{gpt-4o-mini} with a temperature of 0.2. 
    The model was invoked with the following configuration:
    
    \begin{quote}
        \textbf{System Prompt:} ``You are a careful editor. 
        Rewrite the passage into coherent, grammatical English,
        keeping the original meaning and tone. 
        Do NOT add external facts, do NOT invent names, dates, locations, or entities, and do NOT expand abbreviations. 
        Keep unusual or unknown tokens as-is if wrapped in ....
        If ... appears, keep its contents exactly unchanged and keep it in place. 
        Preserve paragraphing; only fix grammar/fluency and minimal function words."\\
        \textbf{User Prompt:} ``Rewrite the passage below. 
        Return ONLY the rewritten passage (no commentary)."
    \end{quote}
    
    \paragraph{Prompts for Answer Generation (Evaluator).}
    To evaluate downstream utility on the SQuAD task, we used \texttt{gpt-4o-mini} as the question-answering model. 
    We used a low temperature ($T=0.2$) to ensure deterministic outputs.
    
    \begin{quote}
    \textbf{System Prompt:} ``You are a helpful assistant that answers questions based on the provided context. 
    Limit your answer to one word.''
    \\
    \textbf{User Prompt:} ``Context: \textit{[Privatized/Repaired Context]}
    \\
    Question: \textit{[Question]}''
    \end{quote}
    
    \paragraph{Fantasy SQuAD Dataset.} 
    A key challenge in evaluating privacy mechanisms with LLMs is data contamination: models like \texttt{gpt-4o-mini} have likely seen the original SQuAD dataset during pre-training and can answer questions from memory even if the context is redacted. 
    To eliminate this confounding factor, we generated a synthetic "Fantasy SQuAD" dataset consisting of fictional passages set in a unique fantasy universe. 
    Because these facts do not exist in the model's parametric memory, the model is forced to rely solely on the privatized context to answer questions, providing a rigorous lower-bound estimate of the mechanism's true utility preservation.
    To ensure the QA task evaluated context usage rather than parametric memory, we synthesized a dataset based on high-fantasy lore. 
    We prompted GPT-4 to create coherent but entirely fictional encyclopedia entries and narrative snippets.

    \begin{quote}
    \textbf{Generation Prompt:} ``Generate a paragraph describing a fictional historical event, city, or biological species that does not exist in the real world. 
    Use unique, invented proper nouns. 
    After the paragraph, provide 5 questions that can be answered \textit{only} by reading the text, along with their extractive answers.''
    \end{quote}

    \noindent This process yielded a dataset where every proper noun and fact is hallucinated by design, ensuring that any correct answer retrieved by the evaluator model must originate from the (privatized) input context.
    
    \paragraph{Additional Examples}

    \begin{figure*}[t]
        \centering
        \includegraphics[scale = 0.15]{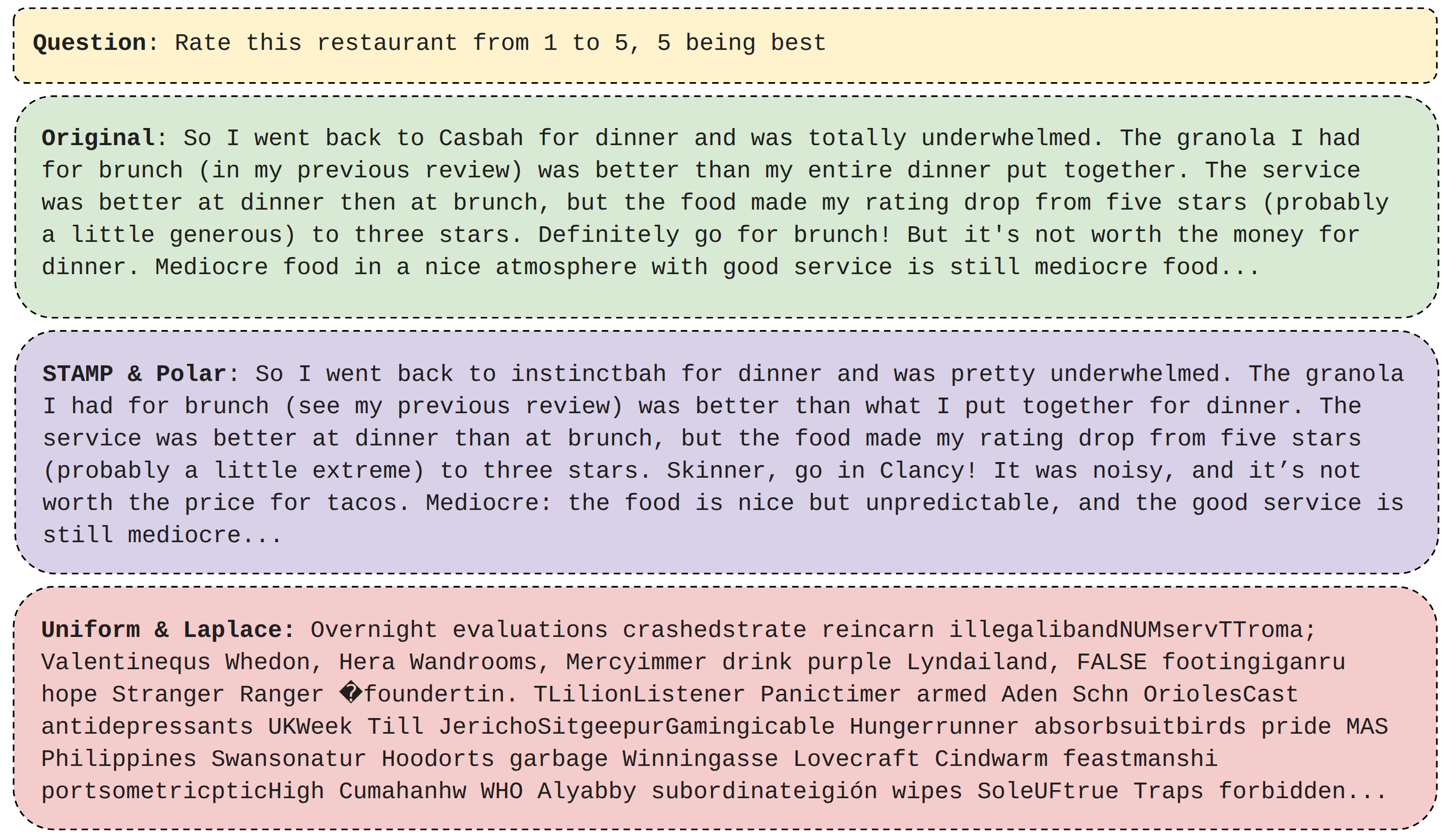}
        \caption{Qualitative comparison on a Yelp review (Sentiment Analysis) at a matched average privacy budget of $\bar{\epsilon} \approx 540$. 
        The privatized outputs shown (middle/bottom) are subsequently coherence-repaired via a constrained rewrite that preserves meaning and forbids adding external facts or inventing names/entities. 
        Even after repair, the \textbf{Uniform \& Laplace} baseline (bottom) collapses semantic structure and remains largely unintelligible. 
        In contrast, \textbf{STAMP \& Polar} (middle) preserves syntactic form and salient sentiment cues (e.g., ``underwhelmed” and the explicit rating ``three stars”), enabling correct sentiment inference despite local perturbations.}
    \end{figure*}

    \begin{figure*}[t]
        \centering
        \includegraphics[scale = 0.15]{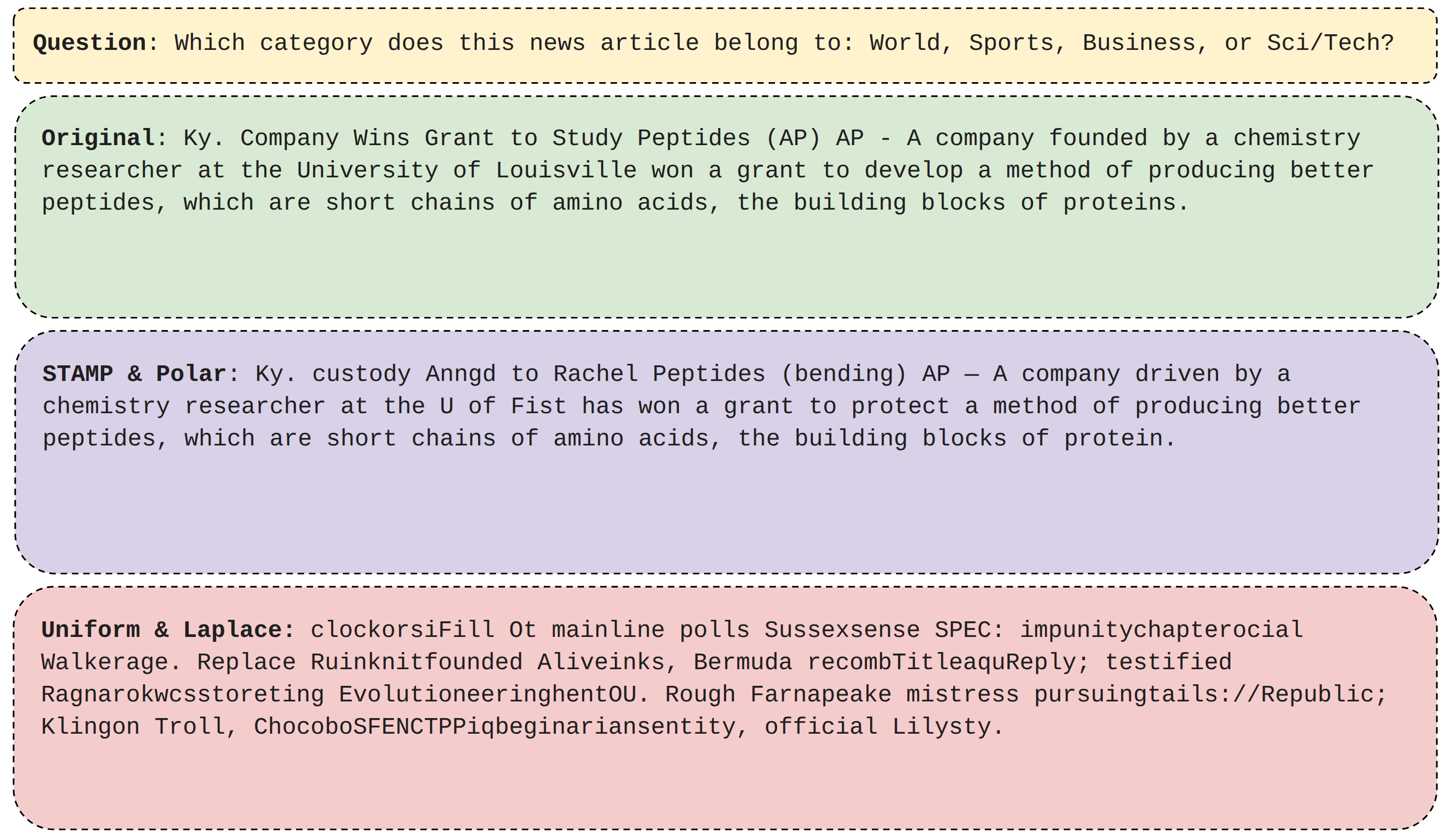}
        \caption{Qualitative comparison on an AG News article (Topic: Sci/Tech) at a matched average privacy budget of $\bar{\epsilon} \approx 490$. 
        The privatized outputs shown (middle/bottom) are subsequently coherence-repaired under the same constrained rewrite setting.
        \textbf{Uniform \& Laplace} (bottom) yields near-complete information loss, producing a sequence of weakly related or unrelated tokens even after repair.
        \textbf{STAMP \& Polar} (middle) successfully retains domain-specific terminology critical for topic classification (e.g., ``Peptides”, ``chemistry researcher”, ``protein”), demonstrating stronger utility preservation under privacy.}
    \end{figure*}

    \begin{figure*}[t]
        \centering
        \includegraphics[scale = 0.15]{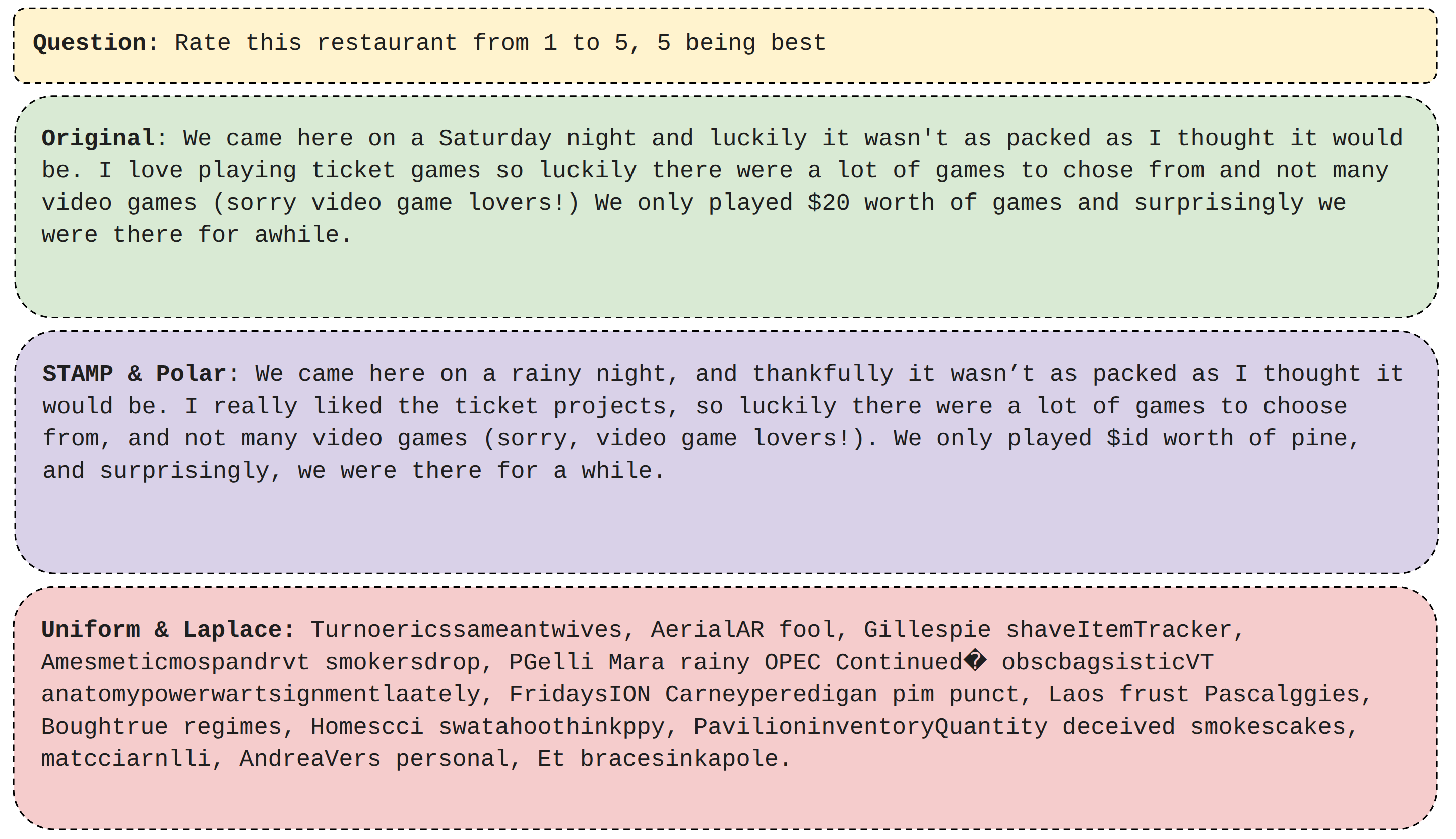}
        \caption{Qualitative comparison on a Yelp review (Sentiment Analysis) at a matched average privacy budget of $\bar{\epsilon} \approx 540$. 
        The privatized outputs shown (middle/bottom) are subsequently coherence-repaired using the same constrained rewrite rule set.
        The \textbf{Uniform \& Laplace} baseline (bottom) results in total semantic collapse, outputting a nonsensical sequence of tokens. 
        Conversely, \textbf{STAMP \& Polar} (middle) maintains narrative structure and key contextual anchors (e.g., ``video games”, ``ticket projects”, ``packed”), enabling the downstream classifier to correctly interpret the review’s positive sentiment.}
    \end{figure*}

    \begin{figure*}[t]
        \centering
        \includegraphics[scale = 0.15]{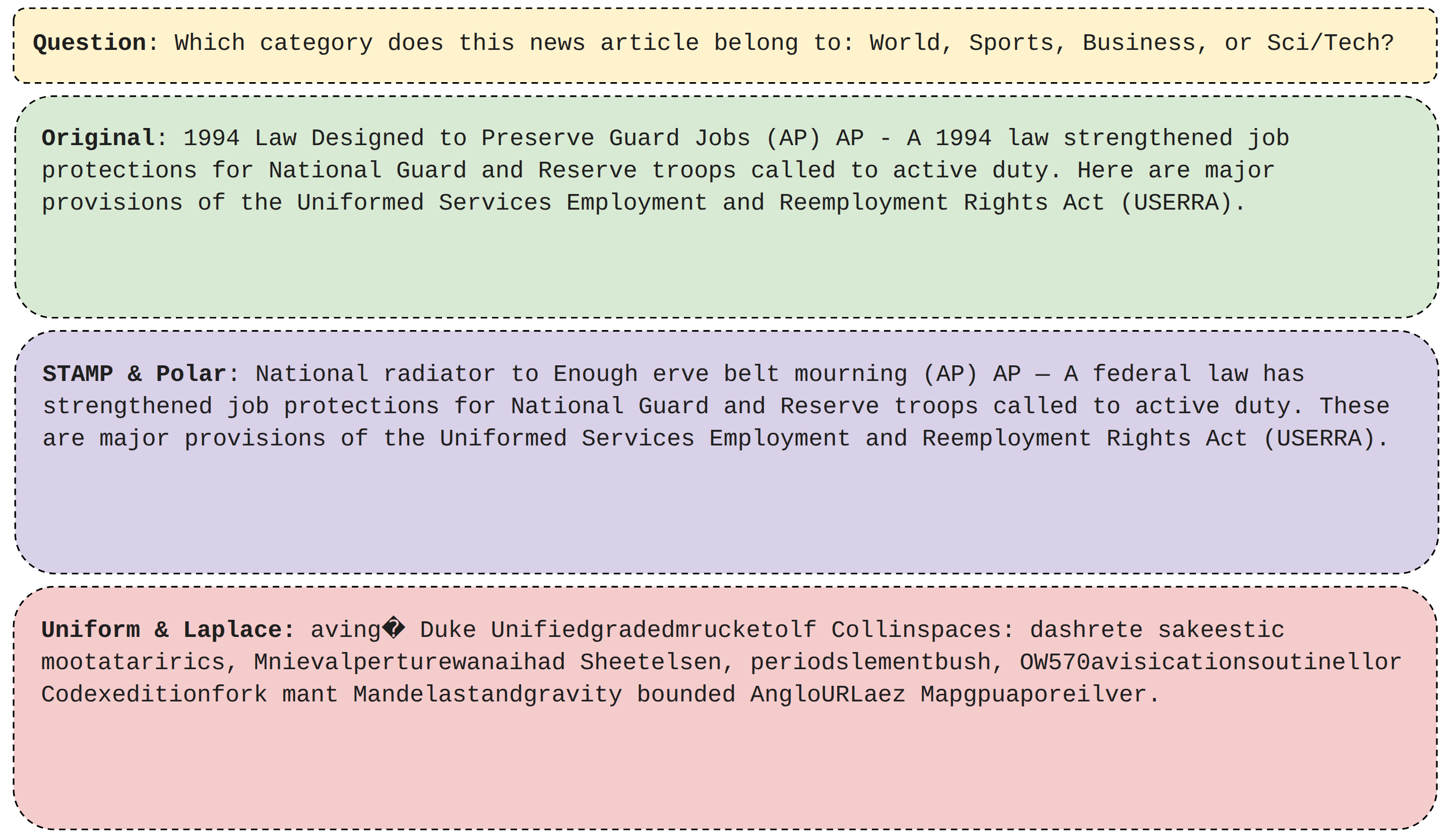}
        \caption{Qualitative comparison on an AG News article (Topic: Business) at a matched average privacy budget of $\bar{\epsilon} \approx 540$. 
        The privatized outputs shown (middle/bottom) are subsequently coherence-repaired via the same constrained rewrite process.
        The \textbf{Uniform \& Laplace} baseline (bottom) results in complete obfuscation, producing a string of unrelated tokens. 
        Meanwhile, \textbf{STAMP \& Polar} (middle) preserves key legal/employment terminology—including ``law strengthened job protections”, ``active duty”, and ``(USERRA)”—supporting correct topic classification under privacy constraints.}
    \end{figure*}

\end{document}